\documentclass[11pt,onecolumn,letterpaper]{article}

\usepackage{times}
\usepackage{epsfig}
\usepackage{graphicx}
\usepackage{amsmath}
\usepackage{amssymb}
\usepackage{soul}
\usepackage{verbatim}
\usepackage{amsmath, amsthm, amssymb}
\usepackage{amsthm}
\usepackage{mdwlist}
\usepackage{mdframed}
\usepackage{booktabs}
\usepackage{enumitem}
\usepackage[ruled,vlined]{algorithm2e}

\usepackage{hyperref}
\hypersetup{colorlinks}

\usepackage{authblk}

\newtheorem{thm}{Theorem}[section]

\newtheorem{lem}[thm]{Lemma}
\theoremstyle{definition}		

\newcommand{\squishlist}{
 \begin{list}{$\bullet$}
  { \setlength{\itemsep}{0pt}
     \setlength{\parsep}{3pt}
     \setlength{\topsep}{3pt}
     \setlength{\partopsep}{0pt}
     \setlength{\leftmargin}{1.5em}
     \setlength{\labelwidth}{1em}
     \setlength{\labelsep}{0.5em} } }

\newcommand{\squishend}{
  \end{list}  }

\definecolor{lightorange}{rgb}{1,0.98,0.95}
\definecolor{darkred}{rgb}{0.5,0,0}

\usepackage{geometry}
\newcommand{\figAwidth}{1.75in}
\newcommand{\figBwidth}{4.5in}
\newcommand{\figCwidth}{3.2in}

\title{Key-Nets: Optical Transformation Convolutional Networks for Privacy Preserving Vision Sensors}
\author[1]{Jeffrey Byrne}
\author[2]{Brian DeCann}
\author[2]{Scott Bloom}
\affil[1]{Visym Labs, Cambridge MA, USA}
\affil[2]{Systems \& Technology Research, Woburn MA, USA}

\date{}

\begin{document}

\maketitle

\begin{abstract}
Modern cameras are not designed with computer vision or machine learning as the target application.  There is a need for a new class of vision sensors that are privacy preserving by design, that do not leak private information and collect only the information necessary for a target machine learning task.   In this paper, we introduce key-nets, which are convolutional networks paired with a custom vision sensor which applies an optical/analog transform such that the key-net can perform exact encrypted inference on this transformed image, but the image is not interpretable by a human or any other key-net.  We provide five sufficient conditions for an optical transformation suitable for a key-net, and show that generalized stochastic matrices (e.g. scale, bias and fractional pixel shuffling) satisfy these conditions.  We motivate the key-net by showing that without it there is a utility/privacy tradeoff for a network fine-tuned directly on optically transformed images for face identification and object detection. Finally, we show that a key-net is equivalent to homomorphic encryption using a Hill cipher, with an upper bound on memory and runtime that scales quadratically with a user specified privacy parameter. Therefore, the key-net is the first practical, efficient and privacy preserving vision sensor based on optical homomorphic encryption.
\end{abstract}

\section {Introduction}

Modern cameras are not designed with computer vision or machine learning as the target application.  Security cameras are designed for professionals performing a forensic video analysis task, such that full images of a scene are collected which contain much more information about a scene than may be necessary for a target computer vision task.  For example, a vision sensor with the task of face detection does not need images of nearby objects in the background, however traditional cameras collect imagery of the entire scene that can reveal much more information than intended.  This is especially true for imagery collected in private spaces such as homes or businesses.  Ideally, a vision sensor is privacy preserving, such that it 
never forms a human recoverable image that could leak information and violate end-user privacy if exposed without consent.

There is a need for novel design for visual sensors that are privacy preserving. Our goal is to replace traditional lens-based imaging system with a new visual sensor designed with novel diffractive or reflective optics, optimized for input to machine learning (ML) algorithms.  Existing cameras have been successfully used for machine learning, however in order to protect privacy, such cameras require digital encryption and are vulnerable to exploitation through third party eavesdropping which risks making private images public. 
Our objective is to develop a coupled vision sensor and ML system that: (i) does not create a human interpretable image, (ii) the ML system can perform inference directly on the sensor measurements, (iii) the ML system is ``keyed'' so that inference can only be performed on observations from the target sensor, (iv) the parameters of the ML system are encrypted and cannot be inspected or repurposed by an adversary and (v) images are encrypted and can only be recovered with knowledge of the secret key physically encoded in the optics.

\begin{figure*}[t!]
\centering
\includegraphics[width=\figAwidth]{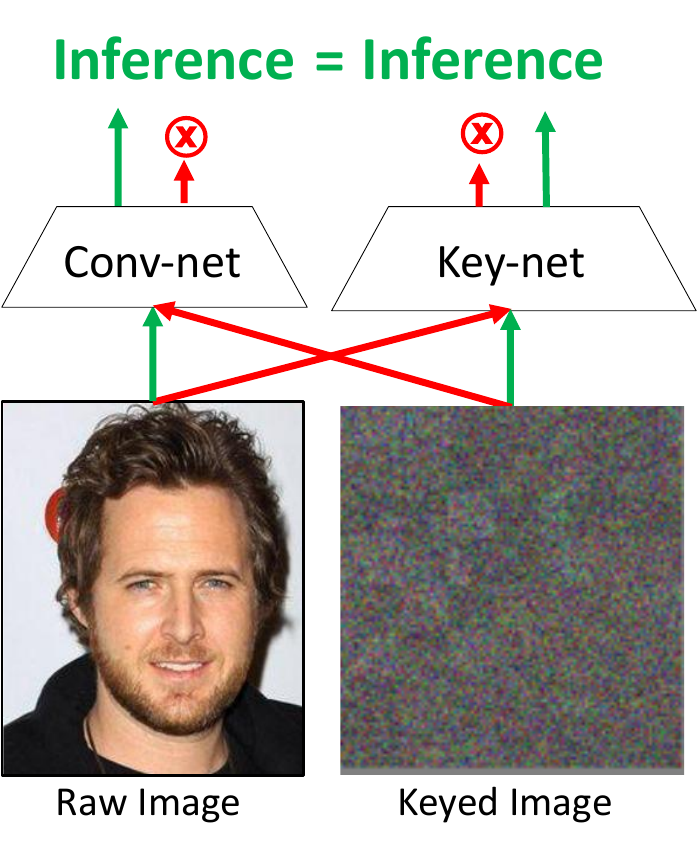}
\includegraphics[width=\figBwidth]{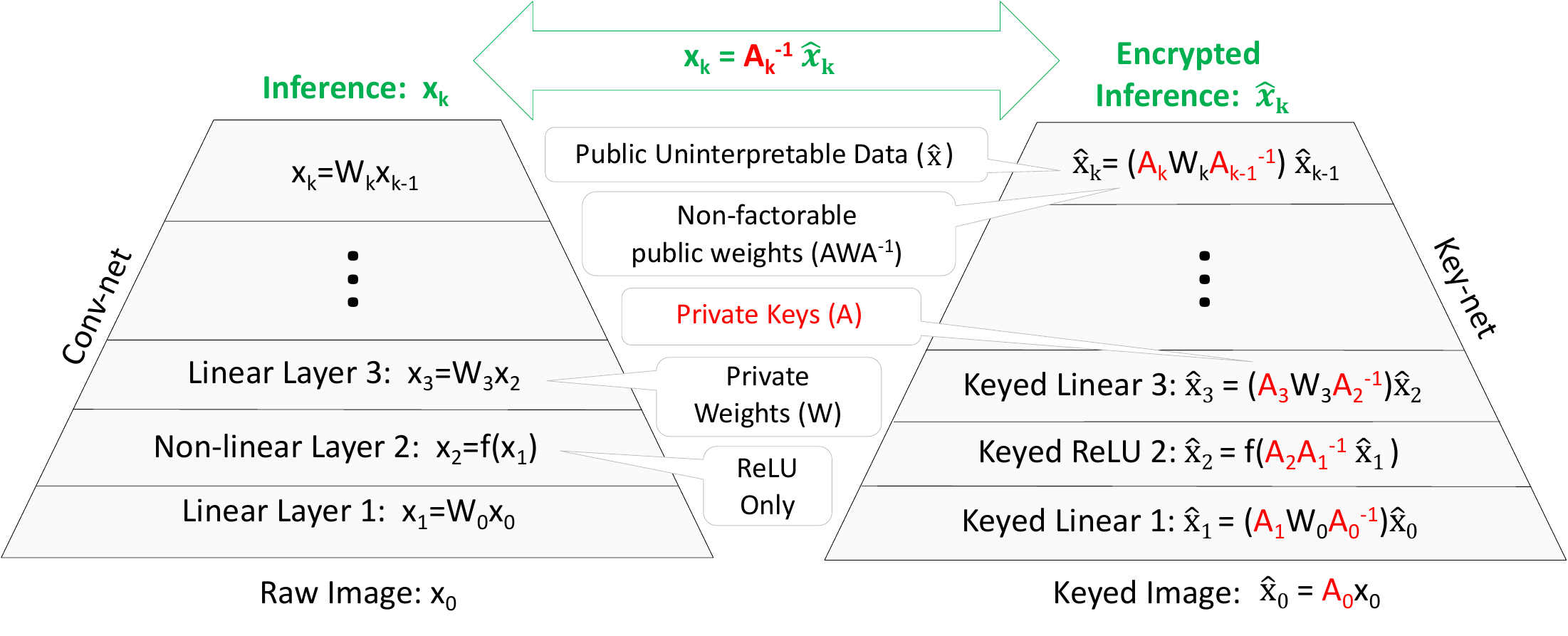}
\caption{(left) A  keynet  is  a  convolutional  network  paired  with  a custom vision sensor which optically transforms an image such that the keynet can perform exact inference on this transformed image.  The observed sensor measurement is not interpretable by a human (or any other keynet) without knowledge of the private key, which is encoded in the optics of the privacy preserving vision sensor.  
(right) A keynet is designed to perform inference on optical/analog encrypted inputs, by generating keyed layers from a source conv-net.  The output of the keynet is equivalent to the encrypted output of the source conv-net, without ever exposing the private image of the scene or the private weights of the conv-net.} 
\label{f:otcn_example}
\end{figure*}

In this paper, we introduce {\em keynets}.  A keynet is the combination of a novel vision sensor and a convolutional network designed specifically for this vision sensor.  The sensor incorporates an optical transformation in the optical imaging chain that transforms an image to be uninterpretable by a human.  However, the convolutional network can be designed to perform inference on this optically transformed measurement without requiring an inversion of the optical transformation.  This coupling of sensor and machine learning system enforces privacy preserving vision sensing because a human interpretable image is never constructed and only the keyed convolutional network can be used for inference on this sensor.  This forms a privacy preserving vision sensor that does not leak personal or private information and cannot be repurposed to another ML task. 

The key contributions of this paper are:

\begin{enumerate}[topsep=0pt,itemsep=-1ex,partopsep=1ex,parsep=1ex]
\item  {\em Optical transformation requirements.}  We describe five sufficient conditions for an optical transformation realizable in the optics of a sensor to enable the design of a keynet from a given source convolutional network.
\item {\em Keynet specification}.  We introduce the design of a keynet to perform inference on the optical/analog transformation of the vision sensor, and show that a feasible family of optical transformations is based on generalized stochastic matrices.  This parameterization includes a user specified privacy parameter ($\alpha$).  
\item {\em Optical element simulation}.  We show that the vision sensor can be physically realized using optical 3D printed fiber bundles and analog gain preprocessing, as shown by end-to-end simulated performance.     
\item {\em Optical homomorphic encryption.}  We show that a keynet is equivalent to a homomorphic encryption based on the Hill cipher \cite{Hill29}, which is physically realized in the optics of the vision sensor.  We describe two well known weaknesses of the Hill cipher and argue that while these weaknesses are present for a generic cryptosystem, they do not introduce a practical risk for a privacy preserving vision sensor.  
\end{enumerate}
\section {Related Work}

The related work can be broadly categorized into three areas:  visual sensors for machine learning, design of novel computational optics and homomorphic encryption.  

First, sensors for machine learning considers redesigning vision sensors for a targeted machine learning task.  An example is the movement of early layers of a convolutional network into the optical  \cite{barua2016direct}\cite{chen2016asp}\cite{Colburn19}\cite{Stork17} or analog processing \cite{Buckler_2017_ICCV}
in order to reduce power consumption. The trend towards optical processing has progressed to consider an all-optical convolutional network \cite{Lin18}, nanophotonic CNN systems \cite{Liu19}\cite{Zhang2019ArtificialNN} or optical and analog CNN hybrids \cite{Chang2018HybridOC} to address the challenges of non-linearities in CNN architectures.  Recently, there has also been work adapting adversarial learning principles for vision sensor design \cite{Chen18}\cite{Li19}\cite{liu2019better}\cite{Mirjalili19}\cite{pittaluga2019learning}\cite{Raval19}\cite{wu2018towards}\cite{xie2017controllable}. 
However, these privacy preserving approaches have demonstrated a clear (and undesirable) privacy/utility tradeoff, such that privacy increases at the expense of primary ML task performance.  Our goal is to enable privacy preserving visual sensing without this privacy/utility tradeoff.
 


The design and fabrication of novel visual sensors considers introducing new computational optics or in-sensor analog computation suitable for a target application.  For example, a coded aperture sensor 
\cite{adams2017single}\cite{antipa2018diffusercam}\cite{Asif17}\cite{boominathan2016lensless}\cite{cannon1980coded}\cite{dicke1968scatter}\cite{fenimore1978coded}\cite{Levin07} replaces the lens with phase masks realized as diffractive optical elements (DOEs) \cite{adams2017single}\cite{Asif17}\cite{sitzmann2018end}, such that imagery can be recovered using computational photography techniques. Such reconstruction-based approaches are also limited by a privacy/utility tradeoff \cite{canh2019deep}\cite{thorpe2013coprime}.  Recent approaches attempt to eliminate the reconstruction task, but are limited by strict camera assumptions and vision tasks \cite{dai2015towards}\cite{davenport2007smashed}\cite{wang2019privacy}\cite{zhao2019active} or design sensors that produce partially human-interpretable images \cite{browarek2010high}\cite{kwan2019deep}\cite{kwan2019multiple}\cite{pittaluga2015privacy}\cite{pittaluga2016pre}\cite{pittaluga2016sensor}\cite{zhang2014anonymous}.
Sensor design has considered angle sensitive \cite{chen2016asp}\cite{Wang2011AnAC} or differential \cite{Wang12} pixels to compute precise motion or angle distribution of the light field, and single photon avalanche diodes \cite{Lee15}\cite{Satat2016LenslessIW} for ultrafast observations. 
Recent work on 3D printing based on two-photon lithography \cite{Ye18} has made possible mass production of custom optical elements at large scales.  Our goal is to leverage this capability to design novel privacy preserving optical elements. 


Homomorphic encryption (HE) \cite{Boddeti2018SecureFM}\cite{Gentry2009FullyHE}\cite{Juvekar2018GAZELLEAL}\cite{Nandakumar2019TowardsDN} is a form of encryption which allows specific types of computations to be carried out on ciphertext and generate an encrypted result which, when decrypted, matches the result of operations performed on the plaintext.  Homomorphic encryption has been applied to convolutional networks to perform computations on encrypted images in: CryptoNets \cite{Chou2018FasterCL}\cite{GiladBachrach2016CryptoNetsAN}, FHE-DiNN \cite{Bourse2017FastHE}, cryptoDL \cite{Hesamifard2017CryptoDLDN}, MiniONN \cite{Liu2017ObliviousNN} and Homomorphic CNNs (HCNNs) \cite{Badawi19}.  These approaches suffer from: inefficient runtime performance, integer discretized weights \cite{Bourse2017FastHE}\cite{GiladBachrach2016CryptoNetsAN}, limited network depth due to increasing noise effects \cite{Badawi19}, polynomial approximations to non-linear activation layers \cite{GiladBachrach2016CryptoNetsAN} or exhibit only partially homomorphic encryption for only additive or multiplicative computations \cite{Brutzkus2018LowLP}\cite{Juvekar2018GAZELLEAL}\cite{Ryffel2019PartiallyEM}.  Our goal is to enable optical homomorphic encryption that does not suffer from these restrictions, but with weaker guarantees on security that we will argue is an appropriate tradeoff for a visual sensor.


\section {Keynets}
\label{s:key_nets}

A {\em keynet} is an optically transformed convolutional network that can perform inference on data collected using a specifically designed sensor.  In this section, we describe requirements for optical transformation ($\S$\ref{s:optical_transformation}) and network construction ($\S$\ref{s:otcn}-\ref{s:stochastic_key_nets}) and optical realization ($\S$\ref{s:optical_realization}).  Privacy analysis is provided in supplemental material ($\S$\ref{s:privacy_analysis}).

Figure \ref{f:otcn_example} shows an comparison of a keynet and a conv-net.  In this example, there is a raw image vectorized to $(x_0)$ which is input to a $k$-layer convolutional network.  This network is composed as $x_k=\mathcal{N}(x_0)$, includes linear and non-linear layers, such that linear layers are represented as a sparse Toeplitz matrix ($W$) and the network outputs inference result $x_k$.   The keynet uses private layer keys $A_i$ to transform the network weights $\hat{W}=AWA^{-1}$, such that the source weights cannot be factored to recover either $A$ or $W$.  The keynet is paired with a custom vision sensor that physically realizes the private image key $A_0$ in an optical and analog transformation chain. Finally, we will show that if the non-linear layers of the source conv-net are limited to ReLU, then the keynet can operate on the transformed input $\hat{x}_k=\hat{\mathcal{N}}(\hat{x}_0)$ and $\hat{x}_k=A_kx_k$.  We call this approach {\em optical homomorphic encryption}.  

Keynets assume the following public and private information.  First, the image key $A_0$ is secret.  The physical sensor containing image key $A_0$ is secret, and controlled with physical security (e.g. in a locked room).  The source convolutional network $\mathcal{N}$ is secret.  The keyed convolutional network $\mathcal{\hat{N}}$ is public.  Optically transformed images ($A_0x$) are public, and raw images ($A_0^{-1}A_0x$) are only recoverable with the secret image key.  Output inference results can be either public or private ($\S$\ref{s:stochastic_key_nets}), and if private can only be recovered knowing the secret embedding key, $A_k$.  
The keyed convolutional network $\mathcal{\hat{N}}$ cannot be used to recover $A$ or $W$, due to the hardness of non-negative matrix factorization ($\S$\ref{s:nmf}).  Therefore, an adversary would be able to observe only the encrypted inference result $\mathcal{\hat{N}}(A_0x)$, an uninterpretable image ($A_0x$) and the keyed layer weights $(\hat{W})$ but not the raw image ($x_0$) or the weights of the source network ($W$) or the raw inference $(x_k)$. 


\subsection {Optical Transformation Function}
\label{s:optical_transformation}


Consider a family of transformation functions $\mathcal{F}$.  A transformation function $f \in \mathcal{F}$ must satisfy the following five sufficient feasibility conditions to be considered an optical transformation function:


\begin{enumerate}[topsep=0pt,itemsep=-1ex,partopsep=1ex,parsep=1ex]
    \item {\em Linear}.  The function $f$ must be linear ($f=A$).
    \item {\em Invertible}.  Matrix $A$ must be positive definite. 
    \item {\em Non-negative}.  $A \geq 0$ for all matrix elements. 
    \item {\em Commutative.}  There exists a non-linear activation function $g$ that is commutative with $A$, such that $A(g(A^{-1}x))= g(AA^{-1}x)=g(x)$.
    \item {\em Sparse.}  Given an $A \in \mathcal{F}$ and $B^{-1} \in \mathcal{F}$, there exists an upper bound such that for any sparse matrix $W$, the product $AWB^{-1}$ is sparse with $|AWB^{-1}|_0 \leq s |W|_0$.
\end{enumerate}

Condition 1 states that the transformation function must be linear, since optical image formation can be modeled as a linear transformation.  Note that this is not a necessary condition, as optical propagation can include non-linear effects due to non-linear optics or diffraction.  Condition 2 states that the transformation is lossless and the original image can be recovered by $A^{-1}Ax$.    
Condition 3 limits a matrix to be physically realizable as a linear optical element  
and closely connects the proposed framework with the computational complexity of non-negative matrix factorization.  Condition 4 enables inference in optically encrypted convolutional networks with non-linear activation function layers.  Finally, Condition 5 ensures that the end-to-end inference in the optically encrypted convolutional network is efficient and does not require the product of an infeasibly large dense matrix.  A family of transformation functions $\mathcal{F}$ is defined to be an optical transformation function if all members of the family satisfy the five feasibility conditions.


\subsection {Optical Transformation Convolutional Networks}
\label{s:otcn}

Consider a convolutional network $\mathcal{N}(x)$ which is the composition of $\mathcal{N}_k$ layerwise functions:
\begin{equation}
\mathcal{N}(x) = \mathcal{N}_k(\mathcal{N}_{k-1}(...\mathcal{N}_1(x)))
\end{equation}
\noindent Given an optical transformation function $A$ and a raw image $x$, $\hat{x}=Ax$ is the optical transformation of the raw image into a sensor observation $\hat{x}$.  First, we will consider the case where $\mathcal{N}$ is linear only, then we will extend to consider a full conv-net with non-linear layers.   

Consider the case where all layers are linear.  In this case, layers $\mathcal{N}_i = W$ are given by a weight matrix $W$ which encodes the linear transformation of a trained convolutional network.  For example, in a typical conv-net, linear layers include convolutional, affine, fully connected, dropout and average pooling layers.  Note that since a convolution is a linear operation, it can be represented as a matrix in the form of a sparse Toeplitz matrix, where the kernel is replicated rowwise.  Furthermore, multi-channel tensor inputs can be flattened to a vector $x$, such that the linear transformation of the layer is the matrix product of a sparse weight matrix and dense data vector.  Finally, note that without loss of generality, a bias $b$ can be applied by projective embedding $x=[x~1]^T$ and affine augmentation $[W~b;0~1]$.   Then, the conv-net simplifies to a matrix product:
\begin{eqnarray}
\mathcal{N}(x;W) = \prod_k W_k x
\label{e:otcn_1}
\end{eqnarray}

\noindent where the notation $\hat{N}(x;W)$ corresponds to network $\mathcal{N}$ with input $x$ and parameters $W$.  Given an optical transformation function $A$, the input $Ax$ can be trivially input to the convolutional network as $\mathcal{N}(A^{-1}Ax)$ by inverting the data prior to inference.  However, this requires exposing the image to the network.  An ideal network would be able to perform inference directly on the optically transformed input, $Ax$, without requiring inversion.    

The linear convolutional network can be constructed to operate on optical transformed inputs as follows.  Linear layers can be replaced by {\em keyed layers} $\hat{W}=AWA^{-1} $ using secret layer keys $A_i$, and a secret image key $A_0$ such that:   
\begin{eqnarray}
\mathcal{N}(x;W) = A_kW_K\ldots (A_2W_2A^{-1}_1)(A_1W_1A^{-1}_0)A_0x 
\label{e:otcn_2}
\end{eqnarray}
\noindent Recall that the keys $A$ and inverse $A^{-1}$ exist by conditions 1 and 2.  By associativity, terms ($A^{-1}A=\text{I}$) cancel, so it follows that (\ref{e:otcn_2}) is equivalent to (\ref{e:otcn_1}).  Furthermore, by associativity, terms can be grouped into the product $\hat{W}=AWA^{-1}$.  Condition 3 requires that elements of $A$ are non-negative, which enables a proof that the factorization of $\hat{W}$ is equivalent to non-negative matrix factorization, which is NP-hard in general ($\S$\ref{s:nmf}).  This protects recovery of $A$ from $\hat{W}$.  Finally, by condition 5, the product $AWA^{-1}$ is sparse, and is at most a factor of $\alpha$ less sparse than $W$.  This bounds the complexity of matrix multiplication $\hat{W}x$ to be at most $\alpha$ times slower than $Wx$.  This enables a practical linear convolutional network operation that preserves sparsity of $\hat{W}$ does not require operations on an impossibly large dense matrix.   

\begin{figure*}[t!]
\includegraphics[width=\textwidth]{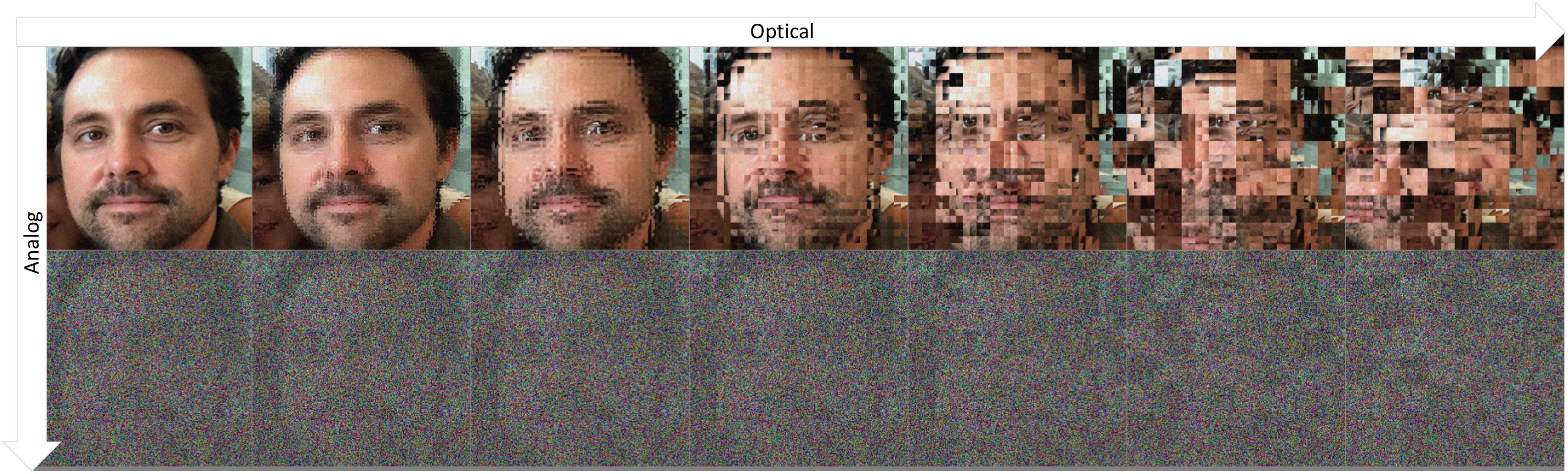}
\caption{Generalized Doubly Stochastic Matrices combine optical and analog processing to form a human uninterpretable image (bottom right), while preserving the flexibility for partially interpretable images (middle column) that is similar to etched ``optical privacy glass''.} 
\label{f:gsm_examples}
\end{figure*}

Let $\mathcal{N}(Ax;\hat{W})$ be the shorthand notation for the {\em keyed network} formed from keyed layers $\hat{W}$ on transformed input $A_0x$, such that: 
\begin{eqnarray}
\mathcal{N}(A_0x;AWA^{-1}) = \prod_k \hat{W}_k A_0x
\label{e:otcn_3}
\end{eqnarray}
\noindent Finally, by associativity, it follows that $A_k\mathcal{N}(x;W) = \mathcal{N}(A_0x;AWA^{-1})$, since (\ref{e:otcn_3}) is shorthand for (\ref{e:otcn_2}) and we previously showed that (\ref{e:otcn_2}) is equivalent to (\ref{e:otcn_1}).  This is a {\em homomorphism}, such that the transformed conv-net output on the original image $x$ is equivalent to the keyed conv-net output on the optical transformed image $A_0x$.

Next, consider a non-linear activation function $g$.  An activation function is a non-linear function typically used in a convolutional network structure to generate the output of a node given inputs.  For example, common non-linear activation layers include rectified linear unit (ReLU), tanh and sigmoid.
By condition 4, we assume that there exists a function $g$ such that $g$ and $A$ are commutative:
\begin{eqnarray}
Ag(A^{-1}\hat{x}) = g(AA^{-1}\hat{x}) = g(\hat{x})
\end{eqnarray}
\noindent This assumption simplifies the non-linear layer to operate directly on the input, which allows the non-linear layer to be included in the keynet without modification.  

Finally, any combination of linear and commutative non-linear layers can be composed into a keynet as follows. 
\begin{eqnarray}
\mathcal{N}_i = \begin{cases}
    \mathcal{N}_i(x_{i-1}, ~A_{i}WA^{-1}_{i-1}) & \text{if~} \mathcal{N}_i \text{~linear} \\
    \mathcal{N}_i(x_{i-1}) & \text{if~commutative non-linear} 
  \end{cases} 
  \label{e:otcn_4}
\end{eqnarray}

\noindent This shows that the keyed network can be constructed with an optical transformation function $A$ to enable the homomorphism $A_k\mathcal{N}(x;W) = \mathcal{N}(A_0x;AWA^{-1})$.  Inference in the keyed network on the optical transformation is equivalent to inference in the source network on the raw image.  This construction enables efficient inference (condition 5) and is secure from recovery of $A$ (condition 3).  Therefore, $\mathcal{N}(Ax;AWA^{-1})$ is a homomorphic encryption scheme for inference of linearly encrypted images in a keyed conv-net.  We call this {\em optical homomorphic encryption}.

\subsection{Generalized Doubly Stochastic Matrices}
\label{s:generalized_doubly_stochastic_matrices}


Section \ref{s:optical_transformation} specified the five conditions for a feasible optical transformation function.  Section \ref{s:otcn} showed that a feasible optical transformation function can be used to construct a convolutional network that operates directly on optical transformed input.  In this section, we show that the family of generalized doubly stochastic matrices satisfies the conditions of an optical transformation function, for choice of activation function $g=\text{ReLU}$.


A doubly stochastic matrix is defined as follows.  First, a permutation matrix or {\em monomial matrix} $\Pi$ is a square matrix that has exactly one entry of one in each row and each column and zero elsewhere, and is constructed by permuting the rows of an identity matrix.
A doubly stochastic matrix is a non-negative matrix such that each row and column sums to one, that encodes a ``soft'' permutation.  It is well known (i.e. the Birkhoff–von Neumann theorem) that every doubly stochastic matrix can be decomposed into a convex combination of permutation matrices.  A generalized doubly stochastic matrix has arbitrary non-zero entries without requiring the rows and columns to sum to one. 
This matrix can be defined as the product of a diagonal matrix $D$ and a doubly stochastic matrix defined as a convex combination of $\alpha$ permutation matrices:
\begin{equation}
    P = \text{D} \sum_{i \leq \alpha} \theta_i \Pi_i
    \label{e:DP}
\end{equation}
\noindent where $\Pi=\sum_i \theta_i \Pi_i$ such that $\sum_i \theta_i = 1$, $\theta \geq 0$.  The convex coefficients $\theta$ are selected to enforce that $\Pi$ is positive definite.  The parameter $\alpha$ encodes the ``softness'' of the stochastic matrix such that larger $\alpha$ is more stochastic, and $\alpha=1$ is equivalent to a permutation matrix.  Furthermore, observe that $D$ can be extended to encode an (optional) additive bias $b$ through an affine augmentation as $[D~b;0~1]$.  The term $\text{D}$ encodes an elementwise multiplicative scaling and additive bias or {\em photometric} degradation, while the term $\Pi$ encodes a pixelwise fractional shuffling or {\em geometric} degradation.

A generalized doubly stochastic matrix satisfies the five conditions of an optical transformation function (\textsection \ref{s:optical_transformation}).  

\begin{enumerate}[topsep=0pt,itemsep=-1ex,partopsep=1ex,parsep=1ex]
    \item {\em Linear.}  $P$ is a linear function as represented by a square matrix $\text{D}\Pi$.
    \item {\em Invertible}.  $P$ is positive definite if and only if both $D$ and $\Pi$ are positive definite.  A sufficient condition for $P$ to be positive definite is selecting $\theta$ in (\ref{e:DP}) such that $\Pi$ is diagonally dominant, and enforcing $\text{diag(D)}>0$.   
    \item {\em Non-negative}. $\Pi$ is non-negative by construction.  If $\text{D}$ is restricted to have strictly positive diagonal entries $\text{diag(D)}>0$, then $P$ is both non-negative and positive definite.  
    \item {\em Commutative}.  Let $g(x)=\text{ReLU}(x)$ and $A$ be restricted to a generalized permutation matrix (e.g. $\alpha=1$ for eq. \ref{e:DP}).  Given this restriction, lemma \ref{p:commutative} in the supplementary material provides a proof of commutativity. 
    \item {\em Sparse}. Given an $\alpha$, there exists a selection of $A \in P$ and $B^{-1} \in P$ such that the product $|AWB^{-1}|_0 \leq \alpha^2 |W|_0$, which is an upper bound on sparsity for $s=\alpha^2$.  Lemma \ref{p:sparsity} in the supplementary material provides proof of this sparsity upper bound. 
\end{enumerate}

\subsection{Stochastic Keynets}
\label{s:stochastic_key_nets}

Section \ref{s:optical_transformation} specified the five conditions for a feasible optical transformation function.  Section \ref{s:otcn} showed that a feasible optical transformation function can be used to construct a convolutional network that operates directly on optical transformed input.  Section \textsection \ref{s:generalized_doubly_stochastic_matrices}) showed that the family of generalized doubly stochastic matrices (e.g. ``soft'' permutation matrices) satisfied the conditions of an optical transformation function, for choice of activation function $g=\text{ReLU}$.  A {\em stochastic keynet} is defined as the selection of doubly stochastic matrices for keying and ReLU for non-linear activation.

Figure \ref{f:gsm_examples} shows examples of generalized stochastic matrices.  The horizontal scale shows optical transformations for increasingly random shuffling due to doubly stochastic matrices.  The vertical scale shows analog transformations for increasingly large gains due to the diagonal matrices.  The combination of these two effects results in a transformed sensor measurement in the bottom right that is uninterpretable to a human observer.  

Construction of a stochastic keynet is as follows:
\begin{enumerate}[topsep=0pt,itemsep=-1ex,partopsep=1ex,parsep=1ex]
    \item Select a pre-trained source conv-net $\mathcal{N}$ that contains only linear and ReLU layers, and a user selected privacy parameter $\alpha$ on $\mathcal{F}_\alpha$.
    \item Randomly select a secret image key $A_0 \in \mathcal{F}_\alpha$.  This is physically realized in the optical and analog imaging chain for a vision sensor as described in section \ref{s:optical_realization}.
    \item If layer $\mathcal{N}_i$ is convolutional, randomly select secret layer key $A_i \in \mathcal{F}_\alpha$.  Convert convolutional kernel to a sparse Toeplitz matrix and keyed layer following (\ref{e:otcn_2}).  If the convolution includes a bias term, perform an affine augmentation of the Toeplitz matrix as: $[W~b; 0~1]$, with projective embedding of input tensor $[x;1]$.  If the layer includes a downsampling stride, the layer keys encode the proper shape.   
    \item If $\mathcal{N}_i$ is ReLU, randomly select secret layer key $A_i \in \mathcal{F}_{\alpha=1}$ such that $A_i$ is restricted to be a scaled permutation matrix.  Transform the input $g(A_iA_{i-1}^{-1}x)$. 
    \item \label{e:output_layer} If $\mathcal{N}_k$ is the output layer, select embedding key $A_k=\text{I}$ if the inference result is public data, else randomly select $A_k \in \mathcal{F}_\alpha$ if the inference is private data.
    \item Compose the stochastic keynet $\mathcal{\hat{N}}(A_0x;~AWA^{-1})$ from $\mathcal{N}(x)$ following (\ref{e:otcn_4}).
\end{enumerate}

The stochastic keynet has two restrictions on allowable conv-net topologies.  The only non-linear layer supported by this network is ReLU or ReLU-like variants (e.g. Leaky-ReLU, Parametric ReLU), as this activation function is commutative with the proposed optical transformation function.  All other non-linear layers are unallowable including: max-pooling, local response normalization (LRN), sigmoid, tanh and softmax.  However, all-convolutional networks have shown that replacing max-pooling with average pooling and limiting activation functions to ReLU maintains state-of-the-art performance \cite{Springenberg15}.  We experimentally validate this claim in section \ref{s:results}.  


Finally, the keynet exhibits a tradeoff between privacy and memory.  A naive Toeplitz matrix construction has $O(N^2K)$ additional parameters than the source network for a layer input tensor of size $(N,N)$ with $K$ channels.  However, these replicated channels are highly compressible.  In our supplemental software, we introduce a sparse matrix format that leverages repeated submatrices as ``tiles''.   In general, the keynet memory requirements scale as $O(\alpha^2 K T)$, given an additional tiling factor $T$ dependent on the sparse matrix storage format.  We show keynet memory requirements as a function of privacy parameter $\alpha$ in section \ref{s:results}.

\subsection{Optical Realization}
\label{s:optical_realization}

The sufficient conditions for an optical transform in section \ref{s:optical_transformation} define a feasible family of transformations for use in a privacy preserving vision sensor.  In the supplemental material (\textsection \ref{s:optical_realization_supplemental}), we show that the selected family of optical transforms based on generalized stochastic matrices can be physically realized using an analog and optical processing chain based on {\em 3D printed incoherent fiber bundle faceplates}.  An optical fiber bundle faceplate is an optical element constructed using a bundle of multi-micron-diameter optical fibers bundled into a thin plate with polished faces.  A simulated example is shown in Figure \ref{f:optical_realization}.


\begin{figure*}[t!]
\includegraphics[width=\textwidth]{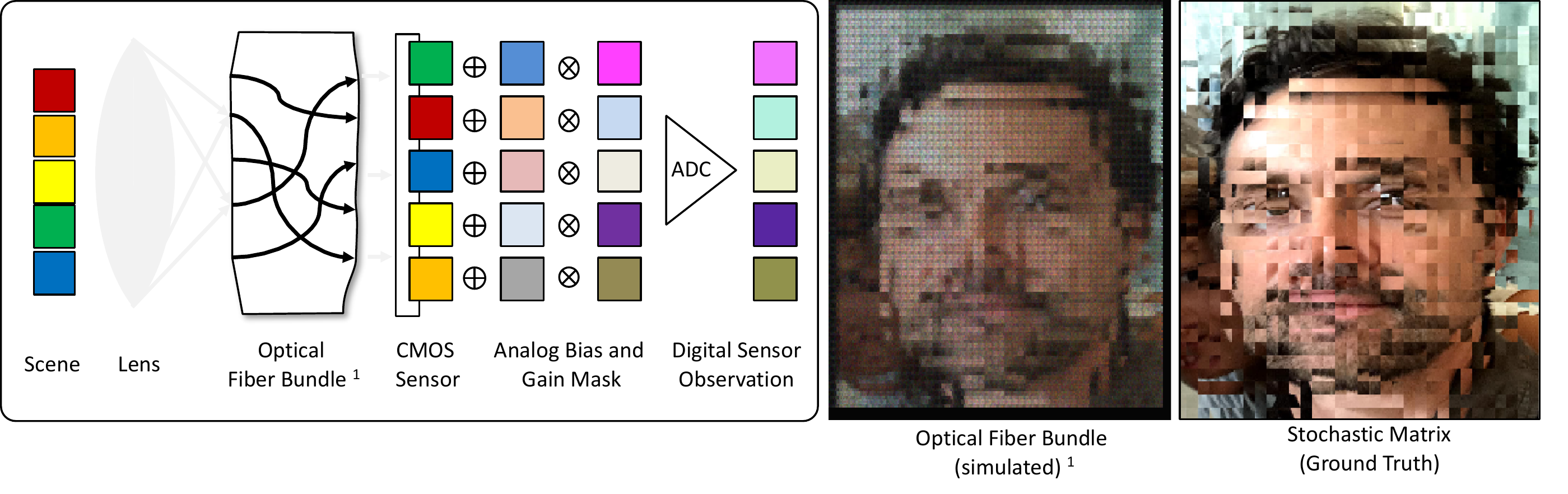}
\caption{Simulation of a 3D printed optical fiber bundle \cite{Ye18} and analog preprocessing to realize a generalized doubly stochastic matrix. } 
\label{f:optical_realization}
\end{figure*}

\section {Experimental Results}
\label{s:results}

A privacy preserving vision sensor must consider the joint design of the sensor and the ML system.  To justify this claim, we consider the following three experiments:

\begin{enumerate}[wide, labelwidth=!, labelindent=10pt]
    \item \label{e:experimental_analysis_1} {\em Frozen System}.  Does there exist an optical transformation that degrades the input image, while preserving performance of a pre-trained ML system?  If such a transformation exists, then a keynet would be unnecessary, since a conv-net could be applied directly to the degraded image, and the degraded image would not be human interpretable.  Experimental results show that the maximum degradation for a pre-trained network to minimize human perception \cite{Rozsa2016AreFA}\cite{Zhou04} while preserving network performance is still clearly human observable.  This provides evidence that preserving image privacy requires joint design of the ML system and the transformation.  See supplemental material (section \ref{s:nullspace}), for detailed results.  
    \item \label{e:experimental_analysis_2} {\em Trained System}.  Can we jointly train an optical transformation and a conv-net to maximally degrade an image while minimizing an ML task loss?  This would also render a keynet unnecessary, as fine-tuning a conv-net on degraded images would suffice.
    Experimental results show that jointly learning an image degradation and a network encoding using an adversarial loss can sufficiently degrade an image to render it uninterpretable by a human.  However, this strategy introduces an undesirable privacy/utility tradeoff where face identification performance degrades by 12\% and object detection degrades by 55\%.  This provides further evidence that preserving image privacy requires keying to preserve the source conv-net performance. See supplemental material (section \ref{s:adversarial_learning} ) for detailed results.  
    \item {\em Keyed System}.  What is the simulated performance of the keynet and proposed vision sensor from Section \ref{s:optical_realization}?  To demonstrate proof of concept of the proposed keynet, we have implemented key-net construction as outlined in section \ref{s:stochastic_key_nets} in PyTorch.  This prototype software exhibits exact inference performance to within floating point error.  Furthermore, we simulated the keynet optical element shown in Figure \ref{f:optical_realization} with simulation strategy described in detail in the supplemental material (section \ref{s:optical_realization_supplemental}).  Results are shown in table \ref{f:keynet_results} for the keynets for three baseline conv-nets.
\end{enumerate}

\begin{figure}[t!]
	\centering
	\includegraphics[width=\textwidth]{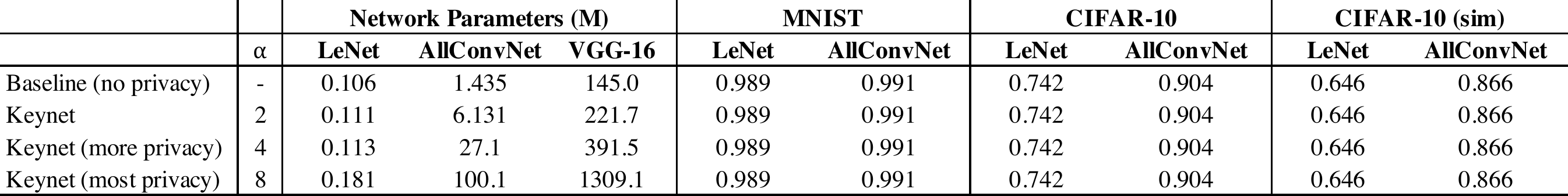}
	\caption{Keynet results.  Model parameters and classification accuracy on raw and optically simulated images for a small (lenet), medium (allconvnet) and large (VGG-16) conv-net and a keynets with increasing privacy $\alpha=\{2,4,8\}$.}
    \label{f:keynet_results}
\end{figure}

\vspace{-0.2cm}
Table \ref{f:keynet_results} shows experimental for three conv-net topologies: 5-layer LeNet, an 11-layer All-Convolutional network \cite{Springenberg15} and a VGG-16 network \cite{Simonyan15}.  All networks were constructed replacing max-pooling with average pooling, as per the keynet requirements.  Results show the keynet memory requirements for a small (LeNet), medium (AllConvNet) and large (VGG-16) conv-net as a function of the privacy parameter ($\alpha$). Naive implementation of the Toeplitz matrices in (\ref{e:otcn_3}) results in an inefficient row-wise replication of the convolutional kernel.  In our supplemental software, we introduce a new tiled sparse matrix format which provides compression of repeated submatrices.  This results in a memory requirement for the keyed network on the order of 4-8x larger than the unkeyed network, depending on the selection of the privacy parameter $\alpha$.  Next, we trained keynets on MNIST \cite{Lecun98} and CIFAR-10 \cite{Krizhevsky09}, using raw images, or simulated optically transformed images from Section \ref{s:optical_realization}.  Results show that the keynet achieves exact inference performance with the baseline, and that the optical simulation results in slightly degraded performance due to minor blurring of the image from fiber cross-talk.  This provides proof of concept in simulation for the keynet optical element.  

\section {Conclusions}
\label{s:conclusions}

In this paper, we introduced keynets, which are the first practical optical homomorphic encryption scheme for the design of privacy preserving vision sensors.  Our experimental results justify next steps which include: a comprehensive study of keynet semantic security as a function of privacy parameter $\alpha$, GPU optimization of sparse tiled matrix-vector multiplication to improve runtime and creation and testing of a prototype optical element.   Keynet software for reproducible research is available for download at \href{https://visym.github.io/keynet}{https://visym.github.io/keynet}.  This includes two prize challenge images and public keynets for attack (\textsection \ref{s:keynet_challenge}) to encourage collaborative discovery of weaknesses in keynet security.

\medskip
\noindent {\bf Acknowledgement.}  This material is based upon work supported by the Defense Advanced Research Projects Agency (DARPA) under Contract No. HR001119C0067.

{\small
\bibliographystyle{ieee}
\bibliography{janus,enigma}
}
\clearpage

\appendix

\section {Supplemental Material}

\subsection{Notation}
\begin{table}[ht]
\begin{tabular}{ll}
$I$ & Input image  \\
$x$ & Input tensor  \\
$\mathcal{F}$ & Optical transform family \\
$\hat{x}$ & Keyed tensor \\
$\hat{x}_0$ & Vectorized optically transformed image \\
$x_0$ & Vectorized raw image \\
$W$ & Layerwise linear transformation \\
$\hat{W}$ & Keyed layer, $AWA^{-1}$ \\
$\mathcal{N}$ & Compositional conv-net function \\
$\mathcal{N}_i$ & Layerwise conv-net function \\
$\mathcal{N}(x;W)$ & Conv-net with input $x$ and parameters $W$ \\
$\mathcal{\hat{N}}$ & Key-net with input $Ax$ and parameters $AWA^{-1}$\\
$A$ & Optical transformation matrix \\
$g$ & Non-linear activation function \\
$|\cdot|_0$ & $L_0$ norm \\
$\alpha$ & User specified privacy parameter \\
$\mathcal{F}_\alpha$ & $\mathcal{F}_\alpha \subseteq \mathcal{F}$, given privacy parameter $\alpha$ \\
$\text{D}$ & Photometric analog elementwise gain and bias \\
$\Pi$ & Stochastic matrix, geometric optical shuffling
\end{tabular}
\end{table}

\subsection{Proof of Commutativity}

\begin{lem}
The function composition $f(g(h(x)))=P^{-1}(\text{ReLU}(Px))$ is commutative for generalized permutation matrix $P=D \Pi$ with permutation matrix $\Pi$ and diagonal matrix $D$, if $\text{D} \geq 0$.
\label{p:commutative}
\end{lem}

\begin{proof}
Let $g(x)=\text{ReLU}(x)$ amd $f(x)=P^{-1}x$ and $h(x)=Px$.  The function composition $f(g(h(x))) = P^{-1}(\text{ReLU}(Px))$ is commutative if the equivalence relation $f(g(x))) = g(f(x))$ holds.  Given a diagonal matrix $D \geq 0$ (i.e. has non-negative entries), the product $P=D \Pi$ for permutation matrix $\Pi$ is non-negative, since permutation matrices are monomial and the product of non-negative matrices is non-negative.  The function $y=\text{ReLU}(x)$ is computed elementwise as $y_i = \mathtt{max}(0,~x_i)$.  Observe that any non-negative scale factor $\beta$ is commutative such that $\text{ReLU}(\beta x_i) = \mathtt{max}(0,~\beta x_i) = \beta y_i = \beta~\text{ReLU}(x_i)$, since a non-negative scaling does not change the sign of $x_i$.  This can be written in matrix notation with $\beta$ on the diagonal of $D$, then $\text{ReLU}(Dx) = D~\text{ReLU}(x)$.
Furthermore, since $\text{ReLU}(x)$ is computed elementwise and $\Pi$ is a one-to-one mapping (e.g. a permutation), $\Pi^{-1} \text{ReLU}(\Pi x) = \text{ReLU}(x)$.  Therefore, $\Pi^{-1}D^{-1} \text{ReLU}(D \Pi x) = \Pi^{-1} \text{ReLU}(D^{-1} D \Pi x) = \text{ReLU}(\Pi^{-1}D^{-1} D \Pi x) = \text{ReLU}(x).$
\end{proof}

\subsection{Proof of Sparsity Bound}

\begin{lem}
Given a sparse matrix $W$ and any $A \in P$ and $B \in P$ where $P$ is the family of generalized doubly stochastic matrices with privacy parameter $\alpha$, there exists a sparsity upper bound $|AWB|_0 \leq \alpha^2 |W|_0$.
\label{p:sparsity}
\end{lem}

\begin{proof}
Let $W_k$ be a sparse matrix with exactly one non-zero element, then the decomposition $W=\sum_k W_k$ such that if $|W|_0=N$ then the decomposition has $N$ terms.  Then, for any conformal matrices $A$ and $B$, $AWB = \sum_k AW_kB$.  Since $A \in P$ and $B \in P$, there exists a decomposition $A=D\sum_i\theta_i\Pi_i$ (resp. $B=D\sum_j\theta_j\Pi_j$).  The sparsity pattern is upper bounded as $|A|_0 \leq |\sum_i\theta_i\Pi_i|_0$, when setting $D=I$.  Each term $AW_iB$ can be expanded into exactly $\alpha^2$ terms of the form $\sum_{i,j} \theta_i \theta_j \Pi_i W_k  \Pi_j$.  The product $\Pi_i W_k \Pi_j$ is a permutation of $W_k$ with sparsity $|\Pi_i W_k \Pi_j|_0=1$.  Therefore, the sum of $\alpha^2$ terms will have at most $\alpha^2$ non-zero elements for every non-zero element in $W$, hence $|AWB|_0 \leq \alpha^2 |W|_0$.
\end{proof}

\subsection{Keynet Example}
\label{s:keynet_example}

Figure \ref{f:otcn_example_supplemental} shows a simple example of a key-net.  In this example, there is a 2x2 raw image vectorized into a $4 \times 1$ vector $(x_0)$ which is input to a two level convolutional network.  This network includes a convolutional layer with kernel $[-1,1]$ (or equivalently a Toeplitz matrix $W_1$), followed by a ReLU layer.
The output of this two layer convolutional network is a vector $[1,0,1,0]^T$.  
The key-net uses private keys $A_1$ and $A_2$ to transform the input and network weights, such that the weights $\hat{W}_1$ cannot be factored to recover either $A$ or $W$.  The key-net operates on the transformed input $\hat{x}_0$ which is observed in a custom designed vision sensor such that $A_1$ is equivalent to a physically realizable optical and analog transformation chain.  Inference in a key-net operates equivalently to the conv-net with transformed weights of the form $\hat{W}=AWA^{-1}$, such that the key-net output is a vector $\hat{x}_2=A_2x_2$.  
This output is equivalent to the conv-net output, encrypted such that $x_2 = A_2^{-1}\hat{x}_2$.  
This is a homomorphism, enabled by an optical transformation $A_1$.

\begin{figure*}
\includegraphics[width=\textwidth]{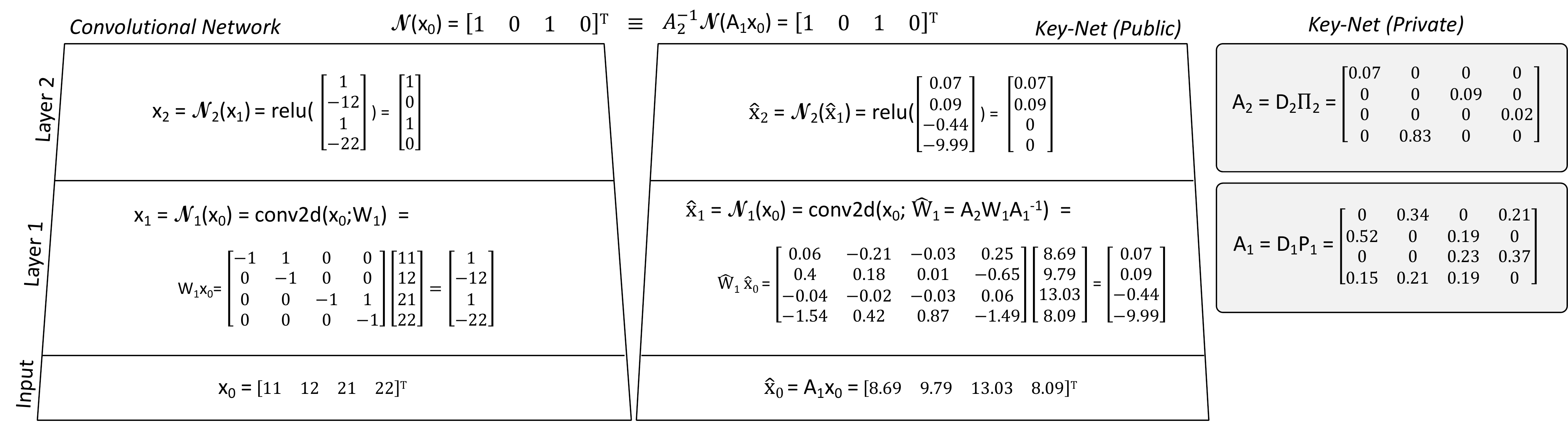}
\caption{Optically Transformed Convolutional Networks. In this example, a 2x2 image $[11,12;21,22]$ is input to a 2 level convolutional network with a convolutional layer and ReLU layer, forming an inference $x_2$. The key-net is constructed from the conv-net using the private keys $A$, such that $A_1$ is a linear transformation implemented in the optics and analog processing of a custom vision sensor forming the sensor observation $\hat{x}_0$. This optically transformed sensor measurement is input to the key-net with output $\hat{x}_2$.  The result of a forward pass in the conv-net is $x_2=A_2^{-1}\hat{x}_2$, however the raw image $x_0$ is never observed or recovered to perform inference in the key-net.  } 
\label{f:otcn_example_supplemental}
\end{figure*}

\subsection{Optical Realization}
\label{s:optical_realization_supplemental}

The sufficient conditions for an optical transform in section \ref{s:optical_transformation} define a feasible family of transformations for use in a privacy preserving vision sensor.  In this section, we show that the selected family of optical transforms based on generalized stochastic matrices can be physically realized using an analog and optical processing chain based on {\em 3D printed incoherent fiber bundle faceplates}.  

An optical fiber bundle faceplate is an optical element constructed using a bundle of multi-micron-diameter optical fibers bundled into a thin plate with polished faces.  
An incoherent faceplate consists of fiber optic strands that are shuffled and rotated so that the faceplace will unfaithfully transmit an image from one face to the other, but in a deterministic manner.  Recent work has demonstrated that optical fiber faceplates can be constructed using 3D printing of thermoplastic filaments \cite{Ye18}.  This enables large scale manufacturing for design of privacy preserving vision sensors.  

Figure \ref{f:optical_realization} shows the design of the optical element to realize a generalized stochastic matrix.  In this design strategy, a lens focuses the light field of the scene onto the optical fiber bundle.  This fiber bundle is designed to implement the doubly stochastic matrix, which shuffles observed pixels and re-transmits them to an alternate location, which is then observed by the CMOS sensor.  Next, during pixel readout, analog preprocessing applies an analog bias and gain.  The resulting pixel readouts are converted from analog to digital (ADC) forming the observed sensor measurement.  The combination of the fiber bundle to implement the stochastic matrix fractional shuffling ($\Pi$) and the analog processing to implement the pixelwise multiplicative scaling and additive bias ($\text{D}$) results in a physical realization of the optical transformation in (\ref{e:DP}). Figure \ref{f:optical_realization} (right) shows the optical simulation of the fiber bundle for a mild permutation, without analog effects for visualization purposes.  Simulation details are described in Section \ref{s:optical_realization_supplemental}.

A 3D printed incoherent fiber bundle faceplate includes the following primary design variables.  First, 3D printed optics are {\em air clad} such that each fiber strand is separated from neighboring strands by an air gap.  This introduces crosstalk due to cladding leakage between fiber strands which reduces the optical transmission fidelity.  Second, 3D printed optics exhibits a {\em minimum fiber diameter} which limits the minimum size of the each optical strand.  This minimum dimension is specified by the diameter of the 3D print head, which is on the order of 100$\mu$m on modern printers.  This is two orders of magnitude larger than a pixel pitch on a CMOS sensor, which requires that each strand covers a pixel neighborhood.  Finally, fiber optic transmission is specified by total internal reflection (TIR), which introduces a cone of projection from the end of the fiber to the CMOS sensor.  This introduces {\em mixed pixels} where the observed intensity is a mixture of the contribution from all neighboring fibers.  Figure \ref{f:optical_realization} (right) shows examples of these modeling errors which must be addressed during sensor calibration.  

In Section \ref{s:results}, we addressed these modeling errors by simulating the fiber optic bundle using parameterization demonstrated by Wang et al. \cite{Ye18} and re-training the key-net to be invariant to these physically realizable effects.   In the remainder of this section, we will describe the simulation of the optical element and the CMOS sensor to simulate the physically realized optical transformation.  

\subsubsection{Optical Simulation}

An optical fiber bundle is simulated as follows. An image of arbitrary size is input to the simulation tools.  The image pixel size is designer defined.  Next a padded mask is defined that is slightly larger than the input image size, the pad size is designer selectable.  The designer then sets the fiber core dimension in the row direction and separately in the column direction.  The simulator allows for a designer defined open area to cladding ratio which allows for image information to be lost due to non-imaging areas in the fiber bundle.  Cladding and fiber core sizes are converted to number of pixels (using pixel size defined above).  A matrix of the centroids of each fiber core in the bundle is initialized.  There is an option for the designer to set a shearing factor which simulates manufacturing tolerances on the array of fiber cores used to form the bundle.  A masking matrix is defined such that for areas of the bundle entered on each centroid matrix element and within the defined fiber core diameter light is transmitted, all other areas are blocked to a designer defined value.  The individual fibers are arranged in a brick-like pattern, i.e. the core centers are offset by one half of the core diameter as the rows go down. The image is then masked with the core and interstitial matrices. The script then rasters through the image and fiber bundle to see which parts of the image fall within allowed fiber cores and which parts are masked.  All parts of the image that fall within a given core are intensity averaged which sets the image resolution to be that of the fiber core size.  Lastly the designer can set crosstalk parameters for both the row and column directions of the bundle which enables the designer to input manufacturing tolerances and/or use of blocking materials between fibers.  The crosstalk value operates like a kernel where the core image intensity is replaced by the vertical crosstalk factor times the sum of the four nearest neighbor vertical core elements plus the horizontal crosstalk factor times the sum of the two nearest neighbor horizontal core elements, normalized to the image maximum pixel value,  for every fiber core in the bundle. This composite image is then taken as the input to the camera noise model defined below.

Figure \ref{f:usaf_simulation} shows an example of this simulation.  User configurable parameters for the fiber bundle simulation are:
\begin{enumerate}[topsep=0pt,itemsep=-1ex,partopsep=1ex,parsep=1ex]
\item Image size
\item  Fiber core size row, column
\item Fiber core/cladding area ratio
\item Fiber bundle shearing factor
\item Fiber interstitial blocking factor
\item Vertical and horizontal fiber crosstalk coefficients
\end{enumerate}

\subsubsection{CMOS Sensor Simulation}

The sensor noise model begins with a given photon intensity hitting a given pixel, this can be set by scaling the input image.  The mean number of photons is given by 
$\mu_{ph}$
and assumes Poisson statistics to calculate the shot noise 
$\sigma^2_{ph}$
(in the limit of large numbers of photons we can use a Gaussian approximation to the Poisson distribution, this should be the case for this system).  The sensor has a defined quantum efficiency depending on the wavelength, sensor materials, and sensor construction geometry, denoted by 
$\nu$
This then gives the number of photo electrons generated in the pixel 
$\mu_e$
which also follow Poisson statistics as:
$\mu_e = \nu \mu_p$.  
Since the statistics are Poisson the variance is also 
$\mu_e$.
In addition to shot noise, the sensor also has a dark noise, i.e. photoelectrons are generated even in the absence of an optical signal.  Usually this dark current is integration time (and temperature) dependent and is due to thermally induced electrons.  The mean dark count is given by a constant term, 
$\mu_0$ 
and the integration time dependent term 
$\mu_I \times t_{int}$
where $\mu_d$
is the sum of these two terms.  Since the thermally induced electrons are also Poisson distributed the dark count variance can be written as:
$\sigma^2_d=\sigma^2_{d0}+\mu_It_{int}$.

Finally, the sum of these sources of photoelectrons charge a capacitor which turns the signal into a voltage, this gets amplified by a gain stage G and then is ultimately converted into a digital signal by the ADC.  This process is assumed to be linear and the camera usually has some over system gain $G_{sys}$ that converts electrons to digital counts out.  The final signal is then given by:
$\mu_{gs} = G_{sys}(\mu_e+\mu_d)$.

Since the signal model is linear and the noise sources are independent, we can RSS the noise sources.  As shown above the readout noise and amplifier noise can be lumped into a dark noise variance
$\sigma^2_d$
and finally there is usually a noise associated with the final ADC stage, 
$\sigma^2_q$
that typically has a uniform distribution and is some camera dependent fraction of the digital scale output. Performing the RSS we get the total camera noise as:
$\sigma^2_{gs} = G_{sys}^2 \sigma^2_d + \sigma^2_q+G_{sys}^2\mu_e$

Adjustable parameters for the camera noise simulation are:

\begin{enumerate}[topsep=0pt,itemsep=-1ex,partopsep=1ex,parsep=1ex]
\item 	Sensor pixel size and number – microns, $N \times N$
\item 	Sensor quantum efficiency - $\nu$
\item 	Sensor dark noise – $\mu_d$, $\sigma_d^2$
\item 	Sensor integration time - $t_{int}$
\item 	Sensor gain on a per pixel basis specified by a gain matrix -  $G_{ij}$
\item 	ADC resolution and noise – depth,  $\sigma^2_{q}$
\end{enumerate}

As an example, the input image and  final output image  of the fiber bundle simulator and camera noise models is displayed in Figure \ref{f:usaf_simulation}.

\begin{figure}[t!]
	\centering
	\includegraphics[width=\textwidth]{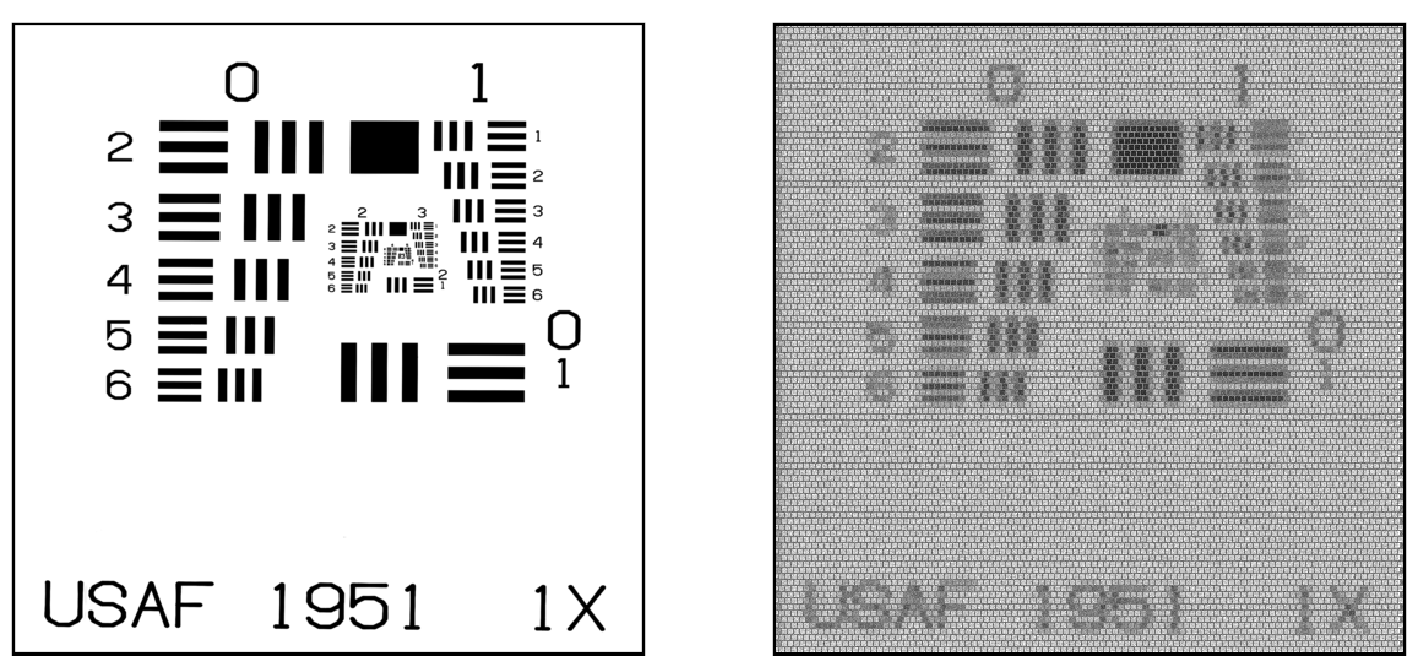}
	\caption{Optical Simulation of the 3D printed fiber bundle.}
    \label{f:usaf_simulation}
\end{figure}


\subsection{Image Transformation Experiments}

A privacy preserving vision system must obscure the image data from the perspective of a human observer, while allowing for vision based tasks.  In this section, we provide ancillary experiments that demonstrate that off-the-shelf convolutional networks (and architectures) are subject to a privacy-performance trade-off, which is a limiting factor compared to the keynet architecture. 

In our experiments, we consider face identification and (later) object detection as the target machine learning tasks.  Face identification experiments were performed using the VGG-16 network architecture, with a pre-trained model \cite{parkhi2015deep}.  Experiments for object detection were conducted using a PyTorch implementation of the Faster R-CNN object detector \cite{RenHGS15} trained on the MS-COCO dataset \cite{Lin2014MicrosoftCC}.  Where applicable, network weights were finetuned using the "training" subset and metrics are reported using the "validation" subset. To report the identification performance of baseline and learned networks, we utilize the Rank-1 classification accuracy.  We acknowledge the limitations of the Rank-1 classification accuracy as a measure of matching performance \cite{matey2015modest}, but for the purposes of this work it is only used as as an indicator of acceptable performance from the matcher, rather than a benchmark.  Object detection performance was measured using average precision over 80 classes, using the standard COCO evaluation metrics \cite{Lin2014MicrosoftCC}.  Human perceptual loss in images is measured by the Structural Similarity Index (SSIM) \cite{wang2004image}.  The motivation for SSIM as a measurement to assess human observable changes in differences is owed to the fact that humans are much more likely to perceive structural changes in images.  Therefore, the SSIM is used as a surrogate for image reconstruction fidelity between a reference (i.e., unmodified, "clean" image) and an optically transformed image.  SSIM is a single-valued measurement and is defined in the range $[0,1]$, where $0$ denotes no similarity (privacy preserving) and $1.0$ denotes the images are equivalent (not privacy preserving). 

\begin{figure*}[t!]
    \centering
    \includegraphics[width=\textwidth]{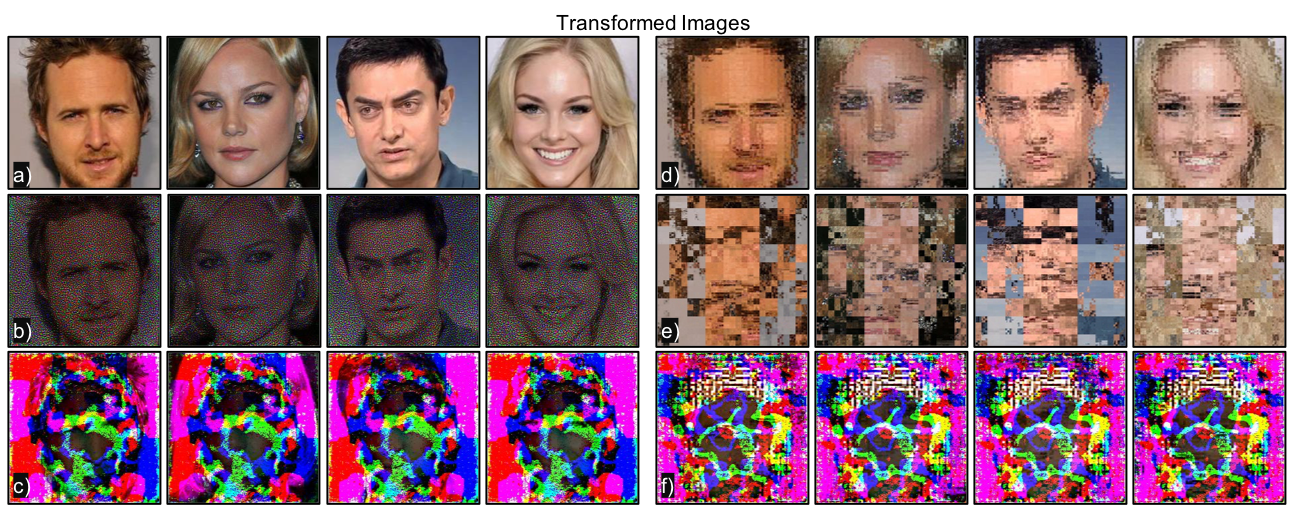}
    \caption{Image transformations. a) Reference images, b) Null-space learning, c) Trained-system learning, d) Geometric (permutation), low; e) Geometric (permutation), high; e) Combined geometric (low) and trained-system. Note that each transformation obscures the image differently.}
    \label{fig:transformations_long}
\end{figure*}



    

\subsubsection{Nullspace Learning for Frozen System}
\label{s:nullspace}

In Section \ref{s:results}, we pose a question asking if it is possible to identify an optical transformation that degrades an input image, while retaining performance of an off-the-shelf (i.e., pretrained) ML system.  In this experiment, we utilize a learned linear image transformation of image $I$ to degraded image $\hat{I}$, with gain and bias parameters $a$ and $b$, respectively.  For simplicity, in this experiment, we restrict learning to only the gain parameter, $a$.  Note that we also performed experiments with bias and gain/bias, which generated similar conclusions.  We also enforce a constraint projection such that $\hat{I}$ is in the integer range of $[0,255]$.

\begin{equation}
    \label{eqn:mask_model}
    \hat{I} = aI + b
\end{equation}

Next, we define an adversarial loss that combines the primary task and an adversarial task, where the primary task is face identification and the adversarial task is image reconstruction (human perception).  This loss measures the performance of the target task relative to the adversarial task, such that this loss is minimized when the target task performance is maximized and the adversarial task performance is minimized.

\begin{equation}
    \label{eqn:loss_fcn}
    loss = L_{primary} + L_{adversarial}
\end{equation}

For a frozen ML system, the weights of the network cannot be modified.  Intuitively, we know that in order to retain matching performance, the convolutional responses of the network must be preserved. Therefore, our primary loss ($L_{primary}$) is defined as the $L_2$ difference of network $\mathcal{N}$ at layer $k$.  The idea here is that if the convolutional responses at layer $k$ are preserved, the downstream network responses will also be preserved.

\begin{equation}
    \label{eqn:null_space}
    L_{primary} = || \mathcal{N}_k(I) - \mathcal{N}_k(I_m)||_2
\end{equation}

The adversarial loss ($L_{adversarial}$) is the compliment of the SSIM function, since our goal is to minimize the SSIM towards zero.

\begin{equation}
    \label{eqn:adv_loss}
    L_{adversarial} = 1 - SSIM(I, I_m)
\end{equation}

The middle-left set of images in Figure \ref{fig:transformations_long} illustrate an example of a nullspace learned image transformation for the face identification task, where the primary task loss was applied at the ``conv5'' layer of the network.  The reference (i.e., unmodified) images are in the upper left.  Note that the transformed images are minimally degraded; they appear darker and with some high frequency noise. This transformation is not privacy preserving.  Metrics from this experiment are a Rank-1 classification accuracy of 0.921 and an SSIM value of 0.065.  These metrics are also listed in the ``Null Space'' row of Table \ref{tab:summary_table} and the training output can be viewed in Figure \ref{fig:rank1_summary}.  We emphasize that for this mask model, we were not able to learn a transformation with a lower SSIM (human perception loss) that achieved reasonable performance at the identification task.



\subsubsection{Trained System Adversarial Learning}
\label{s:adversarial_learning}

\begin{figure}[t!]
	\centering
	\includegraphics[width=\textwidth]{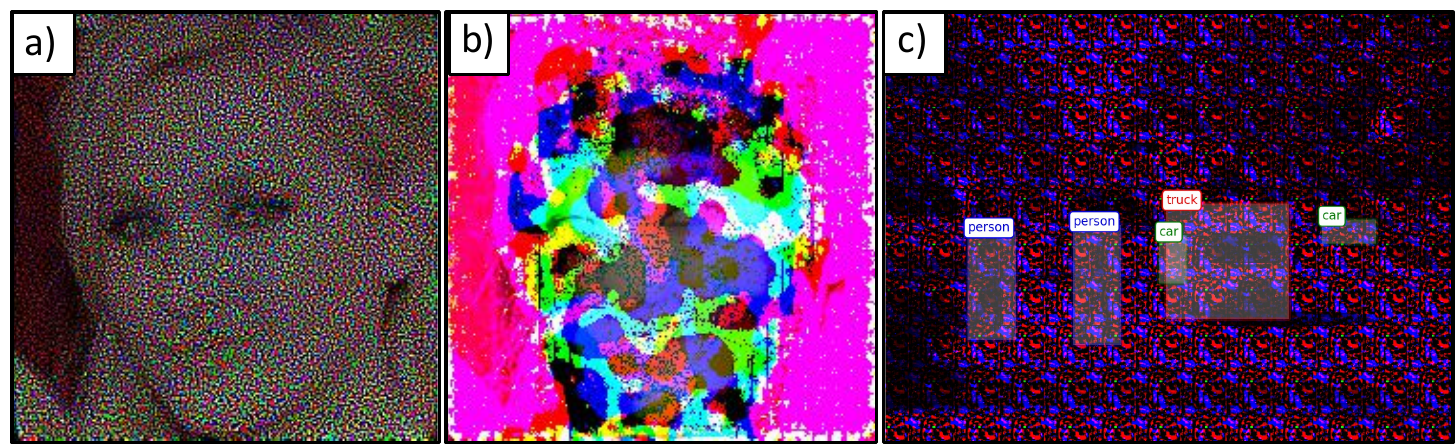} 
	\caption{Adversarial learning examples. (a) Nullspace learning (b) Trained system learning (c) Tiled trained system learning for object detection. These learned degradations exhibit a clear utility/privacy tradeoff for face identification and object detection.}
    \label{f:adversarial_learning_examples}
\end{figure}

The previous experiment provides evidence that there \textit{does not} exist an optical transformation that can degrade an input image, while preserving performance of a pre-trained ML system.  The logical next question is to pose the question asking if such an optical transformation exists if we also jointly minimize a ML task loss (as in $\S$\ref{s:results}, \textit{Trained System}).  If we can accomplish this joint learning task, a key-net would also be unnecessary, as a fine-tuning exercise would be sufficient.  In this section, we report on such an experiment. 

Out joint trained system retains the same mask model as the previous experiment.  That is, the image transformation is linear (\ref{eqn:mask_model}) without a bias parameter.  As with the nullspace learning experiment, we utilize an adversarial loss function, which is now regulated by an $\beta$ parameter and switches between values of $0$ and $1$.  When $\beta=0$, the network weights are updated for the primary task and when $\beta=0$, the transformation weights are updated for the adversarial task.  This approach was considered to ensure degradation of the image would not be skewed toward background content.  Note that for the primary task, the switching point occurred when the cross-entropy loss on the training set decreased below a value of 0.1 (i.e., saturated).  For the adversarial task, the switching point occurred when the Rank-1 classification rate on the validation data decreased below 86\%.  

\begin{equation}
    \label{eqn:loss_fcn_trained}
    loss = \alpha L_{primary} + (1-\beta)L_{adversarial}
\end{equation}

Since in this experiment the weights of the ML system are being learned, we must adjust the $L_{primary}$ loss term accordingly.  For the face identification task, we set this loss term to be the cross-entropy loss function.  The adversarial loss term, $L_{adversarial}$, retains unchanged (\ref{eqn:adv_loss}).  The bottom-left set of images in Figure \ref{fig:transformations_long} illustrate an example of a transformation learned from this approach (the reference images are in the upper left).  Here, we see a considerable increase in the magnitude of the transformation.  Arguably, it may be possible to infer that the example images are of faces, but it is very difficult to deduce the identity, even when provided the reference image data.  However, the trade-off for this transformation is a slightly reduced identification rate.  In this experiment, we achieved a rank-1 accuracy of 0.892, before the adversarial loss saturated with an SSIM value of 0.007.  This data is also reported in the ``Trained System'' row of Table \ref{tab:summary_table} and the training output can be viewed in Figure \ref{fig:rank1_summary}.  Note that the Rank-1 accuracy from training only slightly decreases as a function of training time. This behavior is due to the cyclical nature of the joint optimization.



\subsubsection{Geometric Trained System}
\label{sec:geometric_trained_system}




In the previous two experiments, the image transformation directly modified the value(s) of the image data.  This is not the only mechanism for generating an image transformation.  As described in $\S$\ref{s:generalized_doubly_stochastic_matrices}, we can also create a permutation matrix to ``shuffle'' the image data, which can destroy visual cues for human identification.  Intuitively, we expect that permuting an image would cause a pretrained ML algorithm to fail at its task on this type of data.  In this section, we explore whether it is possible to finetune a pretrained network to perform its primary task, except with permuted image data. 

In this experiment, the actual image transformation is not learned. Instead, the transformation function was carefully crafted to minimize convolutional responses specifically from the VGG-16 network.  The transformation function used permutes blocks of neighboring pixels globally an locally.  The global transformation is a constant translation for each block. The local transformation is a localized permutation within the image block.  This approach also allows us to regulate (or parameterize) the ``amount'' of shuffling that is applied.  Optically, this mask model is a surrogate for a custom optical element utilizing fiber bundles. The top-right and middle-right set of images in Figure \ref{fig:transformations_long} illustrate examples of permuted images using a low-shuffling (top-right) and a high-shuffling (middle-right) approach.  Note that the low-shuffling approach appears similar to a blurring function.  These images are still very human recognizable. The high-shuffling approach however generates images that are not human recognizable.  As with the trained system adversarial learning results, it may be possible to deduce that these (permuted) images are of faces, but it is not possible to infer identity.  We observed that after finetuning the VGG-16 network, the validation accuracy saturated at approximately 94\% and 83\% for the low-shuffling and high-shuffling approaches, respectively.  These metrics are also reported in Table \ref{tab:summary_table} in the ``Geometric (low)'' and ``Geometric (high)'' rows, respectively.  The training output can be viewed in Figure \ref{fig:rank1_summary}.  Note that for the ``high'' geometric transformation, the initial identification performance decreases to zero, but is quickly recovered (up to a point).  Again, we find there is a privacy trade-off between the achieved identification performance and the reconstruction loss.  We hypothesize that this loss in identification performance is due to violating locality of feature data and introducing edges from each local permutation block.

\subsubsection{Combined Trained System}

In this experiment, we combine the trained system approaches in $\S$\ref{s:adversarial_learning} and $\S$\ref{sec:geometric_trained_system}. Here, the finetuned network for a ``low'' geometric transformation is also trained to learn a degraded image (\ref{eqn:mask_model}).  We do not perform this experiment with the ``high'' geometric transformation because the identification performance is too low (83\%) and learning the image degradation would only further reduce performance.  The bottom-right row of images in Figure \ref{fig:transformations_long} illustrates examples of this transformation. This combination of transformations achieved a negligible difference in rank-1 identification accuracy (0.891 vs. 0.892) and SSIM value (0.004 vs 0.007) compared to the trained system approach without geometric permuting of the image data.  These metrics are reported in Table \ref{tab:summary_table} within the ``Combined'' row and the training output can be viewed in Figure \ref{fig:rank1_summary}.

\begin{table}[t]
    \centering
    \caption{Summary of baseline and optical transformation performance for the face identification task. }
    \begin{tabular}[t]{lcc} 
    Experiment & Rank-1 Accuracy & SSIM \\ [0.5ex] 
    \hline\hline
    None (Baseline) & 0.945 & 1.0 \\
    \hline
    Null Space & 0.921 & 0.065 \\ 
    \hline
    Trained System & 0.892 & 0.007 \\
    \hline
    Geometric (low) & 0.940 & 0.339 \\
    \hline
    Geometric (high) & 0.830 & 0.057 \\
    \hline
    Combined & 0.891 & 0.004 \\
    \hline
    \end{tabular}
    \label{tab:summary_table}
\end{table}

\subsubsection{Object Detection}
\label{s:object_detection}

\begin{figure}[t]
	\centering
	\includegraphics[scale=0.35]{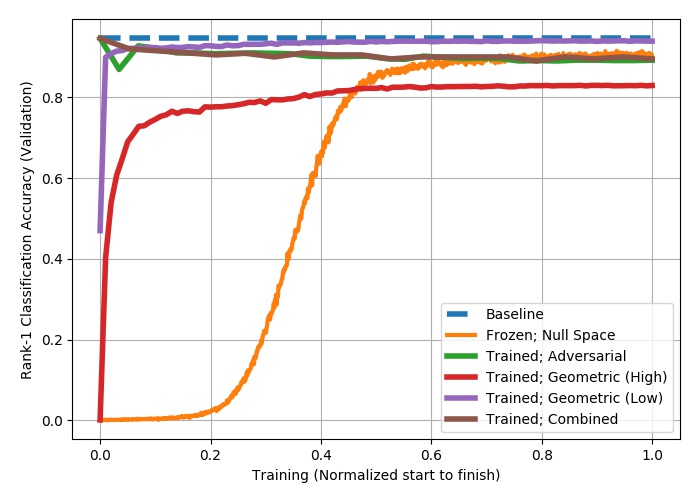}
	\caption{Summary of Rank-1 classification performance achieved for each face identification learning experiment vs. training run-time.  A privacy preserving ML system \textit{must} achieve similar performance to the unmodified baseline (dashed).  Note that \textit{all} of these approaches exhibit some amount of privacy-performance trade-off.} 
    \label{fig:rank1_summary}
\end{figure}

The previous experiments demonstrate that for the face identification task, it is possible to (digitally) apply an optical transformation to an image that strongly reduces human perception (via SSIM) with a small loss in identification performance.  Assuming the optical transformations are physically realizable, the image data is not representative of an end-to-end image acquisition to classification task.  In an end-to-end task, the face data must be detected from a raw, full-scene image, prior to classification.  As such, we performed an experiment to evaluate whether faces could be detected on transformed images.

In our preliminary face detection experiment, we applied the trained system optical transform ($\S$\ref{s:adversarial_learning}) to images in the VGG-Face-1 validation set and executed a face detector.  The face detector was based on the Faster-R-CNN convolutional network and trained to detect faces in natural images.  We observed a 0.0\% detection rate on the transformed images, which suggests that the detector must also be trained for an actual end-to-end system.


Next, we conducted an experiment to extend the trained system adversarial learning to a detection and classification task.  Here, The object detection system is the Faster R-CNN convolutional network, trained on the MS-COCO object dataset.  The primary machine learning task is localization and classification of 80 object classes (e.g. people, vehicles) and the adversarial machine learning task is structural similarity (SSIM) to degrade the image to reduce human perceptibility.  In our experiment, we considered eight total configurations. Each configuration is listed in Figure \ref{f:object_detection_results} (left).  Where denoted, ``tiled gain'' refers to the tiling of pretrained optical transforms from the face detection task. ``Joint gain'' denotes learning of an optical transform using the full scene image.  Evaluation metrics are: AP=average precision, AP for small or large objects only, Relative performance=AP ratio relative to baseline showing performance loss between experimental configurations and baseline, SSIM=structural similarity index. Examples of applied image transforms are illustrated in Figure \ref{f:object_detection_results} (right).  Results show that there is a strong trade-off in detection performance relative to the baseline.  Similar to the null-space experiments for face identification, as it was not possible to learn a full-scene optical transform that did not exhibit considerable performance loss. These results continue to suggest evidence that alternative training strategies are not sufficient for privacy preservation and our key-net approach is required.

\begin{figure*}[t!]
\includegraphics[width=\textwidth]{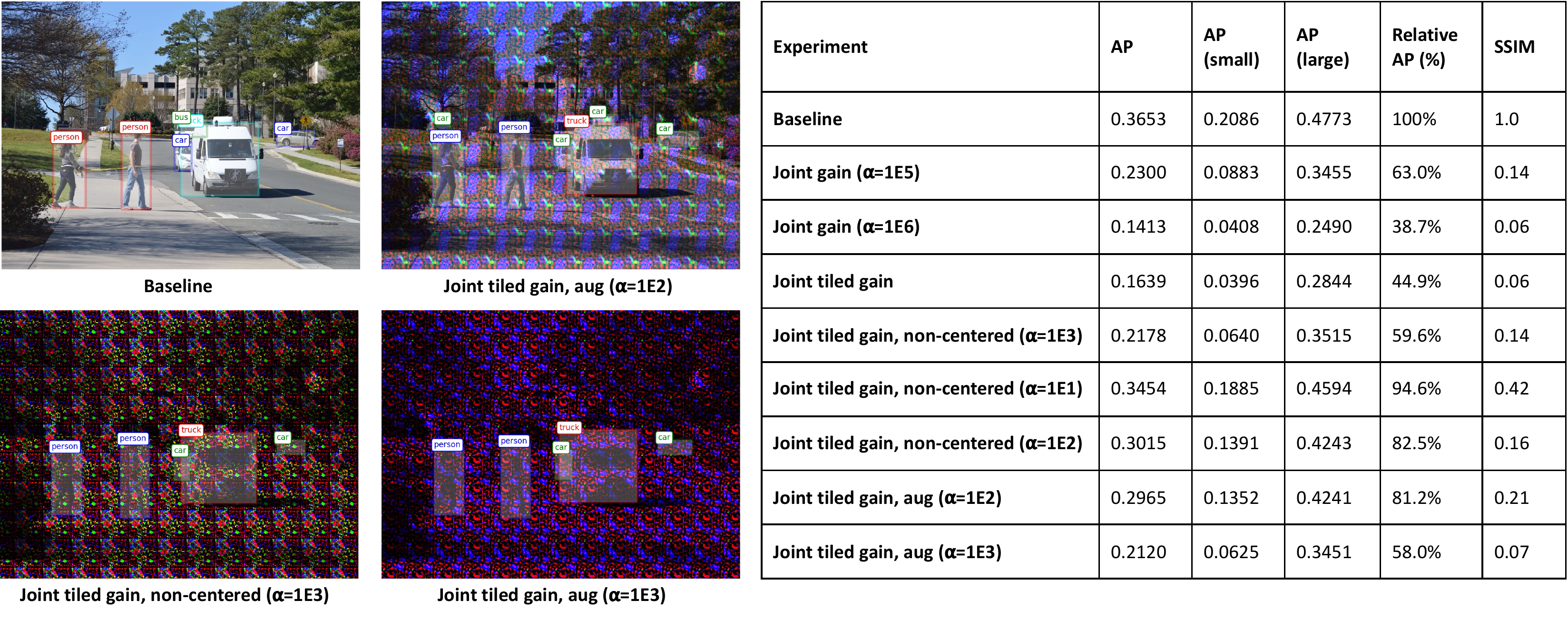}
\caption{Object detection training study.  Results with photometric optical transformation under different tiling and hyperparameter assumptions. See Section \ref{s:object_detection} for details. }
\label{f:object_detection_results}
\end{figure*}

\subsubsection{Summary}

In this section, we performed experiments that justify the necessity of a keynet architecture for a privacy preserving vision sensor ($\S$\ref{s:results}).  These experiments demonstrate that traditional conv-net architectures cannot be refactored to be privacy preserving with indirect (e.g., nullspace learning) or direct (e.g., joint adversarial learning, geometric data permutations) training of the ML algorithm. In each example there is a clear limit and tradeoff on the extent of human perception loss and performance of the primary ML task.  This is evidenced in Table \ref{tab:summary_table}, which reports the achieved Rank-1 accuracy (primary task metric) and SSIM (human perception loss metric) for the face identification task.  In contrast, a key-net does not inherit this privacy-performance tradeoff as its design is fully homomorphic ($\S$\ref{s:otcn}).


\subsection{Privacy Analysis}
\label{s:privacy_analysis}

In this section, we discuss keynet privacy.  First, we connect the problem of recovering source conv-net weights to the problem of non-negative matrix factorization.  Next, we show that the form of encryption we pose is an example of the Hill cipher, a classic cryptosystem based on linear algebra.  Finally, we discuss the primary concern on semantic security, and introduce a challenge problem for the community to analyze it.    

\subsubsection{Non-negative Matrix Factorization}
\label{s:nmf}

Non-negative matrix factorization (NMF) \cite{lee99} is defined as follows.  Given a matrix $V=WH$, factor $V$ into terms $(W,H)$ subject to the constraint $(W,H) \geq 0$, such that elements of the factors are non-negative.  Non-negative matrix factorization in general is NP-hard, with special polynomial time factorizations where $V$ is known low rank.  

Let $AWA^{-1}$ be grouped as $A(WA^{-1})$.  In general, for positive definite matrix $A$ with non-negative entries, the inverse $A^{-1}$ will not be non-negative.  Let $B = (WA^{-1})$, then $B$ can be decomposed
elementwise into the sum of non-negative terms as $B=B_p - B_n$ where $B_{p_i}$=0 if $B_i<0$ else $B_i$ ($B_{n_i}$=0 if $B_i>0$ else $-B_i$, resp.).  Then,
\begin{eqnarray}
    \hat{W} = A(B_p-B_n) \\
    \hat{W} = AB_p - AB_n
\end{eqnarray}
\noindent which transforms the matrix $\hat{W}$ into the sum of products of non-negative matrices.  The elements of $A$ are non-negative by assumption, and the elements of $B_n$ and $B_p$ are non-negative by construction, so then factorization of $AB_p$ or $AB_n$ reduces to non-negative matrix factorization to recover the desired non-negative factor $A$, which can be used to recover $A^{-1}$ and $W$.  An efficient solution to this factorization requires a polynomial time solution to non-negative matrix factorization, for which exact NMF is NP-hard for full rank matrices \cite{Arora2011ComputingAN}\cite{Vavasis2007OnTC}.   Finally, in the case where exhaustive search is possible for ``small'' matrices $V$, NMF in general is non-unique unless further constrained \cite{Huang2014NonNegativeMF}.  So, even if NMF is feasible, the matrix decomposition to recover exactly $A$ is still infeasible.

\subsubsection{Hill Cipher}

The form of optical transformation described is known in the cryptographic literature as a Hill cipher \cite{Hill29}.  The Hill cipher is a classic cryptosystem based on a linear transformation matrix as secret key.  Transformed images ($Ax$) are robust to cryptanalysis and can be safely made public, as long as the key $A$ is kept secret.  Furthermore, as described in section \ref{s:nmf}, the product $AWA^{-1}$ is also secure to known ciphertext attacks, due to the hardness of non-negative matrix factorization.  This enables public disclosure of both optical transformed images and key-nets, while ensuring security of raw images and source network weights.

However, the Hill cipher exhibits a two known weaknesses in the form of {\em chosen plaintext} and {\em chosen ciphertext} attacks. In a chosen plaintext attack scenario, the unknown $A$ can be recovered through least squares regression with at least $N$ tuples $(x, Ax)$, for $A$ with known sparsity $|A|_0=N$.  However, this requires that the attacker has physical access to the sensor, and in this scenario, privacy has already been compromised.  
The sensor can be assumed to be locked in a private space such as the home, with physical access restricted to authorized users, so tuples $(x,Ax)$ cannot be collected by policy.  

The Hill cipher also exhibits a weakness to {\em chosen ciphertext attack}. In this attack scenario, the adversary is provided decryptions $A^{-1}y$ of a chosen ciphertext $y$.  Like the chosen plaintext attack, the unknown $A^{-1}$ can be estimated using least squares regression.  However, the key-nets will not be used in this scenario by design, as the image does not require decryption and the output inferences can be public.   
So, while the Hill cipher does have a weakness as a general cryptosystem, we believe it is an appropriate and practical assumption for a privacy preserving vision sensor.  

Finally, the most challenging requirement is proving {\em semantic security}.  Semantic security is the problem of exposing information about the plaintext given only the ciphertext.  For example, in a key-net consider the case where the optical transformation function is the identity matrix.  The resulting key-net is exactly the source network, and the encrypted images are identical to the raw images.  Clearly, this provides no security.  A more subtle challenge for semantic security is when the optical transformation is a diagonal matrix or a permutation matrix. In section \ref{s:semantic_security}, we discuss that these transformations exhibit a semantic security weakness, which exposes the structure of $\hat{W}$ to attack.  We discuss that using the generalized stochastic matrix with privacy parameter $\alpha > 1$ shows promise to defend against this attack.

\begin{figure*}[t!]
\includegraphics[width=\figCwidth]{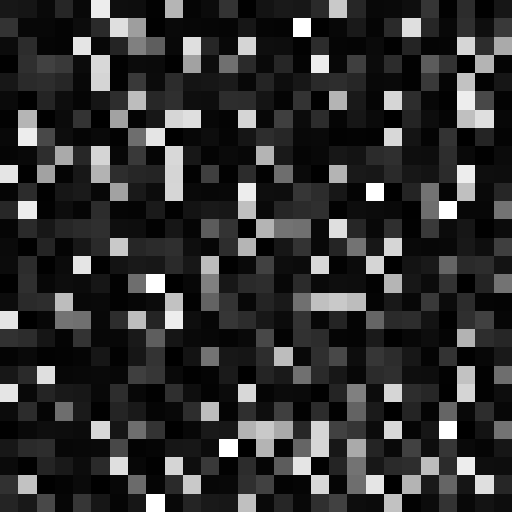}
\includegraphics[width=\figCwidth]{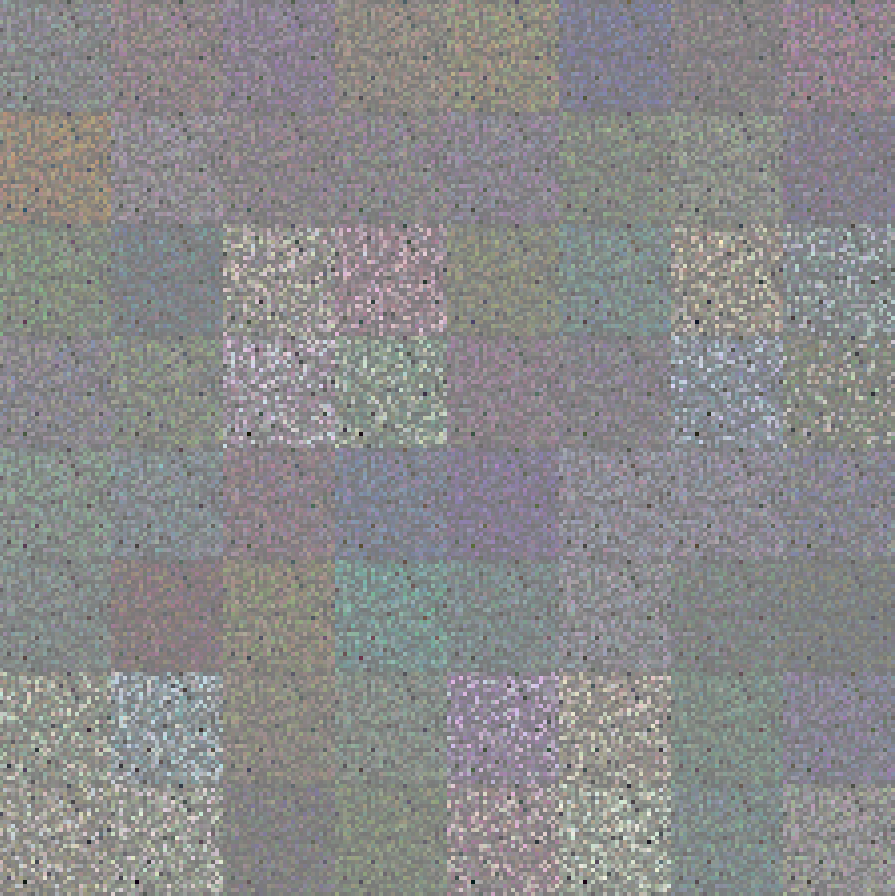}
\caption{Keynet challenge problem.  These images contain a secret message.  We will release these images along with their paired keynet to challenge the research community to discover a weakness in semantic security of our proposed approach.}
\label{f:keynet_challenge}
\end{figure*}

\subsubsection{Semantic Security}
\label{s:semantic_security}

Semantic security is the problem of exposing information about the plaintext given only the ciphertext.  A subtle challenge with semantic security considers the case where the degradation is either a scaling or a permutation rather than a generalized stochastic matrix.  In the case of a scaling, the weakness leverages natural image statistics for recovery, such that gradients are sparse for neighboring pixels (e.g. images are smooth almost everywhere).  For example, blind deconvolution techniques with Total Variation (TV) regularization can be used in some cases to jointly recover the unknown degradation kernel and the original image.  For key-net attacks, the concept is to leverage the distribution of sparse gradients in natural images, which can be used to regularize this ill-posed problem to recover the unknown image mask and raw image.  Future work will consider these different optimization strategies to determine conditions for which image reconstruction using this strategy is feasible. 


A second subtle challenge for semantic security is when the optical transformation is a permutation matrix.  In this case, the neighborhood structure of a convolution is present in the non-zero structure of the Toeplitz matrices in the key-net.  The keyed layers of the key-net are public information, so the sparsity structure of the weights can be inspected and used by an attacker.  For example, there exists a greedy optimization based on graph embedding to recover the structure of a permuted image with known neighbors simply by minimizing the pairwise embedding distance of pixels.  This is analogous to ``puzzle solving'', with the simplification that puzzle piece neighbors are observable in the sparsity pattern in the Toeplitz matrices which implement keyed convolutions.  This is not a risk if the key-net is kept private, but if the key-net is public, then $\hat{W}$ exposes private information about $Ax$.  Introducing the privacy parameter $\alpha$ can mitigate this attack by making the neighborhood structure ambiguous by increasing the sparsity of $W$ by a user specified privacy factor that is independent of the true neighborhood structure.  This introduces a tradeoff between inference runtime/memory and privacy that mitigates this attack.  Furthermore, combining the permutation with an analog scaling and bias results in limiting the attack due to natural image statistics.  Future work will investigate the feasibility of this style of attack for key-net images as a function of $\alpha$.

Finally, the conditions listed in Section \ref{s:optical_transformation} are sufficient, but not necessary.  Future work will explore alternative selections of the image key $A_0$ that are positive semi-definite.  In this case, the sensor observation cannot be inverted to recover the image, even under a plaintext attack, since the least squares optimization is under determined.  In the key-net framework, we would set $A_0^{-1}=I$ and continue the key-net encoding as currently described.  This would further protect against semantic security attacks, but would likely introduce a utility/privacy tradeoff which would degrade the trained ML task performance as $A_0$ becomes increasingly rank deficient.

\subsubsection{Challenge Problem}
\label{s:keynet_challenge}

Finally, we plan on publicly releasing the challenge images in figure \ref{f:keynet_challenge} and associated public key-nets for a lenet and vgg-16 topology.  These challenge images contain a secret message that can only be discovered by exploting a weakness in semantic security.  We would like to encourage the community to collaborate to discover such weaknesses in our approach by sponsoring a prize challenge.  These images and public keynets are available for analysis at \href{https://visym.github.io/keynet}{https://visym.github.io/keynet}.

\end{document}